\documentclass[twoside]{article}
\usepackage[accepted]{aistats2019}
\usepackage{hyperref}
\usepackage{relsize}
\usepackage[utf8]{inputenc} 
\usepackage[T1]{fontenc}    
\usepackage{url}            
\usepackage{booktabs}       
\usepackage{amsfonts}       
\usepackage{nicefrac}       
\usepackage{microtype}      
\usepackage{amsmath}
\usepackage{amsthm}
\usepackage{amssymb}
\usepackage{xcolor}
\usepackage{algorithm}
\usepackage{algorithmic}
\usepackage{thm-restate}

\newlength\myindent
\newcommand\bindent{%
  \begingroup
  \setlength{\itemindent}{\myindent}
  \addtolength{\algorithmicindent}{\myindent}
}
\newcommand\eindent{\endgroup}

\definecolor{graphicbackground}{rgb}{0.96,0.96,0.8}
\definecolor{rouge1}{RGB}{226,0,38}  
\definecolor{orange1}{RGB}{243,154,38}  
\definecolor{jaune}{RGB}{254,205,27}  
\definecolor{blanc}{RGB}{255,255,255} 
\definecolor{rouge2}{RGB}{230,68,57}  
\definecolor{orange2}{RGB}{236,117,40}  
\definecolor{taupe}{RGB}{134,113,127} 
\definecolor{gris}{RGB}{91,94,111} 
\definecolor{bleu1}{RGB}{38,109,131} 
\definecolor{bleu2}{RGB}{28,50,114} 
\definecolor{vert1}{RGB}{133,146,66} 
\definecolor{vert3}{RGB}{20,200,66} 
\definecolor{vert2}{RGB}{157,193,7} 
\definecolor{darkyellow}{RGB}{233,165,0}  
\definecolor{lightgray}{rgb}{0.9,0.9,0.9}
\definecolor{darkgray}{rgb}{0.6,0.6,0.6}
\definecolor{babyblue}{rgb}{0.54, 0.81, 0.94}
\definecolor{citrine}{rgb}{0.89, 0.82, 0.04}
\definecolor{misogreen}{rgb}{0.25,0.6,0.0}

\DeclareMathOperator*{\argmin}{arg\,min}

\DeclareMathOperator{\Tr}{Tr}

\newcommand{\ceil}[1]{\left\lceil#1\right\rceil}

\newtheorem{corollary}{Corollary}

\newtheorem{remark}{Remark}
\newtheorem{lem}{Lemma}
\newtheorem{prop}{Proposition}

\newcommand{\R}{\mathbb{R}}

\newcommand{\pa}[1]{\left(#1\right)}

\newcommand{\abs}[1]{\left|#1\right|}

\newcommand{\inner}[2]{\langle #1, #2 \rangle}

\newcommand{\norm}[1]{\left\lVert#1\right\rVert}

\newcommand{\CommaBin}{\mathbin{\raisebox{0.5ex}{,}}}

\newcommand{\transpose}{^\mathsf{\scriptscriptstyle T}}
\newcommand{\eqdef}{\triangleq}

\newcommand{\cB}{\mathcal{B}}

\newcommand{\cD}{\mathcal{D}}

\newcommand{\cL}{\mathcal{L}}

\newcommand{\cM}{\mathcal{M}}
\newcommand{\cN}{\mathcal{N}}
\newcommand{\cO}{\mathcal{O}}

\newcommand{\cP}{\mathcal{P}}

\newcommand{\cR}{\mathcal{R}}

\newcommand{\cS}{\mathcal{S}}

\newcommand{\bc}{{\bf c}}

\newcommand{\bM}{{\bf M}}

\newcommand{\bU}{{\bf U}}

\newcommand{\bV}{{\bf V}}
\newcommand{\bw}{{\bf w}}

\newcommand{\bx}{{\bf x}}
\newcommand{\bX}{{\bf X}}

\newcommand{\eps}{\varepsilon}
\renewcommand{\epsilon}{\varepsilon}
\renewcommand{\hat}{\widehat}
\renewcommand{\tilde}{\widetilde}

\newcommand{\nothere}[1]{}

\usepackage{xspace}

\newcommand{\malocate}{\normalfont \textcolor[rgb]{0.5,0.2,0}{\texttt{MALocate}}\xspace}
\newcommand{\esterr}{\normalfont \textcolor[rgb]{0.5,0.2,0}{\texttt{EstimateError}}\xspace}
\newcommand{\getest}{\textcolor[rgb]{0.5,0.2,0}{\texttt{GetEstimator}}\xspace}
\newcommand{\newdata}{\textcolor[rgb]{0.5,0.2,0}{\texttt{NewSamples}}\xspace}

\usepackage{wrapfig}
\usepackage{subfigure}

\usepackage[round]{natbib}

\usepackage[disable]{todonotes}
\definecolor{babyblue}{rgb}{0.54, 0.81, 0.94}
\definecolor{citrine}{rgb}{0.89, 0.82, 0.04}
\definecolor{misocolor}{rgb}{0.16,0.27,0.86}

\definecolor{blued}{RGB}{70,197,221}
\definecolor{pearOne}{HTML}{2C3E50}
\definecolor{pearTwo}{HTML}{A9CF54}
\definecolor{pearTwoT}{HTML}{C2895B}
\definecolor{pearThree}{HTML}{E74C3C}
\colorlet{titleTh}{pearOne}
\colorlet{bull}{pearTwo}
\definecolor{pearcomp}{HTML}{B97E29}
\definecolor{pearFour}{HTML}{588F27}
\definecolor{pearFith}{HTML}{ECF0F1}
\definecolor{pearDark}{HTML}{2980B9}
\definecolor{pearDarker}{HTML}{1D2DEC}
\hypersetup{
	colorlinks,
	citecolor=pearDark,
	linkcolor=pearThree,
	breaklinks=true,
	urlcolor=pearDarker}

\begin{document}
\twocolumn[

\aistatstitle{Active multiple matrix completion with adaptive confidence sets}

\aistatsauthor{Andrea Locatelli\And  Alexandra Carpentier  \And Michal Valko}

\aistatsaddress{OvGU Magdeburg  \And OvGU Magdeburg   \And  SequeL team, Inria Lille}]

\begin{abstract}
In this work, we formulate a new multi-task active learning setting in which the learner's goal is to solve multiple matrix completion problems simultaneously. At each round, the learner can choose
from which matrix it receives 
a sample from an entry drawn uniformly at random.
Our main practical motivation is 
\emph{market segmentation}, 
where the matrices represent different
regions with different preferences of the customers.
The challenge in this setting is that each 
of the matrices can be of a different size and also 
of a different rank which is unknown.
We provide and analyze a new algorithm, 
\malocate that is able to adapt to the unknown ranks of the different matrices.
We then give a lower-bound
showing that our strategy is minimax-optimal and demonstrate its performance with synthetic experiments.
\end{abstract}

\section{Introduction}
In this work, we consider the setting of completing multiple matrices in a sequential and active way, under a budget constraint on the number of observations the learner may request. The learner's objective is to estimate each of these matrices well (in some precise sense that we define later) and is akin to the \textit{pure exploration} problems considered in the multi-armed bandits~\citep{bubeck2011pure,gabillon2011multi}. As the learner is trying to solve multiple learning problems simultaneously, a decent strategy should naturally allocate a larger portion of the observational budget to harder problems. Such challenge is for example considered in a very different model by~\cite{riquelme2017active}. Of course, since knowing the hardness or \emph{complexity} of each instance is typically out of reach in practice, a good strategy should be \emph{adaptive} to the different complexity scenarios, without requiring any tuning. This is in contrast with previous results for regret minimization with a low-rank structure~\citep{katariya2016stochastic,katariya2017bernoulli}, where the learner explicitly takes advantage of the rank-$1$ structure of the setting.

We consider matrix completion in the \textit{trace-regression model}
\citep{klopp2012noisy,rohde2011estimation,koltchinskii2011nuclear,negahban2012restricted}. 
There are important
reasons regarding this choice
as opposed to the \textit{Bernoulli model} \citep{candes2009exact,chatterjee2015matrix}, 
another common model for the matrix completion.
In particular, in the trace-regression
model it is possible that some 
of the matrix entries are sampled multiple times.
In the Bernoulli model, this cannot happen, as each entry is observed either never or once with probability~$p$ in the simplest model. 
The implication of this \textit{multi-sampling}
is fundamental as it allows, in the trace-regression model, to construct
honest confidence sets that \textit{adapt to the rank} of the matrix, even if the level of noise is unknown. On the other hand, it has been shown that in the Bernoulli model such confidence sets provably do not exist \citep{carpentier2017adaptive}. This is very important, as we will see that our adaptive strategy crucially depends on the existence of these adaptive confidence sets: Consider for example the problem of minimizing the maximum of the losses across multiple matrix completion problems. A good strategy should roughly equalize the diameter of the confidence sets across instances when the budget expires, as it pays the price for the largest diameter by definition of the maximum loss. In order to do that, it is important to leverage adaptive confidence sets.

The main application domain we target is 
market segmentation \citep{wedel200market} and polling.
However, being able to multi-sample decides the situations
where exactly this model applies. For example, 
for music recommendations in music streaming services,
it is possible that the users listen to
the same song twice or more and we can 
get multiple samples of their 
appreciations, either by rating or by not-skipping.
For movie or product
ratings, multi-sampling is much less applicable. Yet it is possible to ask the customer for a second opinion later in time. In other situations, the multi-sampling happens
by design. For example,in tasting experiments, the human subjects are sometimes
given same two  samples, that they have to taste and evaluate with a week-long break in between.
Our algorithm and results apply to these situations, whether the multiple-sample for the same entry
are possible because of the nature of the setting 
or by design.

In this work, we introduce the active multiple matrix completion problem and propose an anytime algorithm (\malocate) that solves this problem \emph{adaptively} to the unknown ranks of each sub-problem. For the $\max$ loss, which corresponds to the case where the learner pays the price of the largest loss on the set of matrix completion problems it has to solve, we show that our strategy is optimal by deriving a matching lower bound. Finally, we show that \malocate indeed performs well with a synthetic experiment.

\section{Multiple matrix completion setting}\label{s:setting}
We start by defining the single matrix completion problem and state the known results that we build on. Then, we introduce our active setting, which can be thought of as solving $K$ matrix completion problems simultaneously (as the objective is to optimize the loss when the budget~$n$ expires) and sequentially as we may decide where to allocate our budget at round $t \leq n$.
\subsection{Single matrix completion setting}
We first introduce the matrix completion setting and a matrix lasso estimator. Let $\bM_0 \in \mathbb R^{d_1 \times d_2}$ be an unknown matrix. The task of matrix completion is that of estimating $\bM_0$ accurately in some precise sense, that we define later, by an estimator $\widehat{\bM}$ given $n$ independent random pairs $(\bX_i,Y_i)_{i\leq n}$ such that
\begin{equation*}
Y_i \triangleq \Tr(\bX_i\transpose \bM_0) + \sigma\epsilon_i,
\end{equation*}
where the $\epsilon_i$ are centered independent random variables with unit variance.\footnote{In this paper, we will restrict ourselves to the case of bounded noise, but our results can be extended to sub-exponential noise as in the work of~\cite{klopp2012noisy}.} We consider the matrix completion setting where  $\bX_i\transpose$ are i.i.d.\,uniformly distributed on the set
\begin{equation*}
\mathcal X \triangleq \left\{e_i\left(d_1\right)e_j\transpose\left(d_2\right), i\in [d_1], j  \in[d_2] \right\}\CommaBin
\end{equation*}
where $e_i(d)$ are the canonical basis vectors in $\mathbb R^d$. Typically, in this setting, we do not observe the entire matrix of size $d_1\times d_2$ as we have $n \ll d_1 d_2$, and we consider matrices of low rank $r$, with respect to $\min(d_1,d_2)$, for which completion is still possible despite the low number of observations. Let $d \triangleq \max(d_1, d_2)$ and $\norm{\bM}_F$ is the Frobenius norm of a matrix $\bM = (\bM_{ij}) \in \R^{d_1 \times d_2}$ defined as \[{\norm{\bM}_F^2 \triangleq \sum_{i=1}^{d_1}\sum_{j=1}^{d_2} \bM_{ij}^2 = \Tr(\bM\transpose \bM)}.\]  For this problem, it is possible to construct good estimators $\widehat{\bM}_n$ such that
\begin{equation*}
\frac{\|\widehat{\bM}_n - \bM_0\|_F^2}{d_1 d_2} \leq \rho(r,n,d),
\end{equation*}
where $\rho(r,n,d) \ll \norm{\bM_0}_\infty$ for $r \ll \min(d_1, d_2)$ and $n \geq r d$. Intuitively, the higher the rank $r$ of $\bM$, the harder the problem should be, as there are more parameters to estimate. A good estimator should be \textit{adaptive} to the rank of the matrix without requiring it as an input to allow the tuning of hyperparameters.
\subsection{Square-root lasso estimator}
In this work, we consider the matrix square-root lasso estimator, which has been shown to have favorable 
properties (\citealp{candes2006near-optimal, klopp2012noisy,gaiffas2011sharp,koltchinskii2011nuclear}). We define the nuclear norm of a matrix\[\norm{\bM}_\star \triangleq \Tr\pa{\sqrt{\bM\transpose\bM}} = \sum_{i=1}^{r} \sigma_i,\] where $\sigma_i$ are the singular values of~$\bM$. The matrix square-root lasso estimator is defined as
\begin{equation}\label{eq:min}
\!\!\!\!\!\widehat{\bM}_n(\lambda) \!\in \!\!\!\argmin_{\bM \in \R^{d_1 \times d_2}}\!\! \left\{\!\!\sqrt{\sum_{i=1}^n\! \frac{(Y_i \!- \!\inner{\bX_i}{\bM})^2}{n}}\! +\! \lambda \|\bM\|_\star \!\!\right\}\cdot
\end{equation}
Importantly, for this estimator \cite{klopp2012noisy}
showed that $$\rho(r,n,d) = \cO\pa{\frac{rd\log d}{n}}$$ for~$\lambda$ defined in the following proposition, that  \emph{does not depend} on $r$, the unknown rank of matrix $\bM$. It also does not require the variance $\sigma^2$ of the noise as an input to tune $\lambda$, only an upper bound such that $A \geq \sigma$.

\begin{prop}[upper bound, \citealp{klopp2012noisy}]\label{thm:ub}
There exist numerical constants $c$ and $C$ such that with probability at least $1-3/d-2\exp(-cn)$, the matrix square-root  lasso estimator $\widehat{\bM}_n$ satisfies
\[
\frac{\|\widehat{\bM}_n - \bM\|_F^2}{d^2} \leq \frac{CA^2 \cdot rd\log d}{n}\CommaBin
\]
where $\widehat{\bM}_n$ is defined as the solution to the minimization problem in Equation~\ref{eq:min}
with $\lambda \eqdef C'A\sqrt{(\log d)/(nd)}$ where~$C'$ is a numerical constant.
\end{prop}

We also restate a lower bound for the single matrix completion problem  shown by \citet[Theorem 5]{koltchinskii2011nuclear}, which shows that the previous procedure is minimax optimal up to an extra $\log d_k $ factor.

\begin{prop}[lower bound, \citealp{koltchinskii2011nuclear}]
For any estimation procedure that outputs $\widehat{\bM_n}$ from~$n$ noisy observations corrupted with independent noise $\epsilon_t \sim \cN(0, A^2)$, there exists a matrix $\bM$ of size $(d \times d)$ and rank at most $r$ such that
\[
\mathbb E\left[\frac{\|\widehat{\bM}_n-\bM\|_F^2}{d_1 d_2}\right] \geq \frac{cA^2rd}{n}\CommaBin
\]
where $c$ is a small numerical constant and the expectation is taken with respect to both the distribution of the samples and the possible internal randomization of the estimation procedure.
\end{prop}
This result easily extends to the bounded noise case.

\subsection{Adaptive confidence sets}
An important theoretical result in the trace-regression model with uniform sampling of the entries is the existence of \emph{adaptive and honest} confidence bands on the error $||\widehat{\bM} - \bM||_F^2$. Importantly, the knowledge of $\sigma$ is again not necessary for this estimator. This procedure, \esterr, is described in Section~\ref{s:alg}, and makes use of the entries $X_i$ that have been observed twice to compute an unbiased estimator of the error. This procedure comes with the following guarantee.
\begin{prop}[concentration bound for $\widehat{R}_N$ estimator, \citealp{carpentier2017adaptive}] 
Let $N$ be the number of entries that have been observed twice in the second half of the sample and $\widehat{R}_N$ be the (unbiased) estimation procedure (sub-procedure \esterr) of $\|\widehat{\bM}-\bM\|_F^2$, for some $\widehat{\bM}$. Then with probability at least $1-\frac{2}{d}\CommaBin$ we have
\[
\abs{\widehat{R_N} - \frac{\norm{\widehat{\bM}-\bM}_F^2}{d^2}} \leq 8A^2\sqrt{\frac{\log d}{N}}\cdot
\]
\end{prop}
For minimax-optimal estimation procedures, such as the square-root lasso, we can show (by bounding both the estimation error as above and $N \geq Cn^2/d^2$ for some numerical constant $C$, on a favorable event) that with high probability,
\[\widehat{R_N} + 8A^2\sqrt{\frac{\log d}{N}} \leq \cO\pa{\frac{rd\log d}{n}}\CommaBin\] 
which shows that this quantity is an \textit{adaptive} (as it does not require the rank as an input) and \textit{honest} (as it upper bounds the true error with high probability) confidence band on $\|\widehat{\bM} - \bM\|_F^2.$ 

\subsection{Active multiple matrix completion}
In the active multiple matrix completion, the learner's goal is to complete multiple matrices $\{\bM^k\}_k$ simultaneously, by
\emph{actively choosing} from which matrix it should ask for a new observation in a sequential and adaptive manner. For ease of notation, we restrict this setting to square matrices of dimension $d_k$, but our techniques directly extend to non-square matrices. At each round the active learner has to choose an action $k_t \in [K]$ and receives a pair $(\bX_t^{k_t},Y_t^{k_t})$ 
such that $\bX_{t}^{k_t}$ corresponds to the location of the \emph{entry} $(i_{k_t,t}, j_{k_t,t})$ of the $k_t$-th data matrix $\bM^{k_t} = (\bM^{k_t}_{ij}) \in \R^{d_{k_t} \times d_{k_t}}$ chosen uniformly at random such that $i_{k_t,t} \in [d_{k_t}]$ and $j_{k_t,t} \in [d_{k_t}]$, and 
\begin{align*}Y_t^{k_t} & \triangleq \Tr(e_{i_{k_t,t}}(d_{k_t})e\transpose_{j_{k_t,t}}(d_{k_t}) \bM^{k_t}) + \epsilon_{k_t,t} \\&= \bM_{i_{k_t,t} j_{k_t,t}} + \epsilon_{k_t,t},\end{align*}
where the $e_i(d)$ are the canonical basis vectors of~$\R^d$. Here, $\bX_t^{k_t} = e_{i_{k_t,t}}(d_{k_t})e\transpose_{j_{k_t,t}}(d_{k_t})$. Informally, the learner chooses to observe one of the $K$ matrices, and receives a noisy observation  of one of the entries (corrupted by $\eps_{k_t,t}$) chosen uniformly at random from that matrix. The goal of the learner is to adaptively choose which matrix $\bM^{k_{t}}$ to sample based on the observations collected so far up to round $t-1$,
\[\left\{(\bX_{1}^{k_1}, Y_{1}^{k_1}),\dots, (\bX^{k_{t-1}}_{t-1}, Y^{k_{t-1}}_{t-1})\right\}\cdot\]
At the end of the game, once it has collected at most~$n$ pairs $(\bX_t^{k_t},Y_t^{k_t})$, the learner has to output estimates $\hat{\bM}^k_n$ of each matrix $\bM^k$ to suffer the following loss,
\begin{align*}
\cL_n^p \triangleq\left(\sum_{k\in [K]} \|\hat{\bM}^k_n - \bM^k\|_F^{2p}\right)^{\!\!\!1/p}\!\!\!\!\!,
\end{align*}
where $p$ characterizes the objective and is decided as part by the learner at the start of the game. As special and interesting cases, for $p=1$, we recover the unnormalized squared Frobenius norm if the sub-problems were the blocks of a block-diagonal matrix, and for $p= \infty$ the max loss $\max_{k\in [K]} \|\hat{\bM}^k_n - \bM^k\|_F^2$.
\begin{remark} As an extension, we can consider the re-weighted loss, characterized by a given weight vector $\bw = (w_1, ..., w_K)$, where $w_i \in \mathbb R^+$ for $i \in [K]$ is a parameter given to the learner along with p, 
\[\cL_n^p(\bw)  = \left(\sum_{k=1}^K w_k \|\hat{\bM}^k_n - \bM^k\|_F^{2p}\right)^{\!\!\!1/p}\!\!\!\!\!.\] 
Taking $w_k = d_k^{-2}$ allows to consider the normalized Frobenius norm for each matrix, which is particularly interesting in combination with $p = \infty$ as it is simply the maximum average loss per entry within each sub-problem, regardless of the dimension.
\end{remark}

For each matrix $\bM_k$, $k \in [K]$, we  denote by $r_k$, the rank of $\bM^k$. We further assume that all the observations~$Y_{t}^{k_t}$ and the entries of $\bM^k$ are bounded by some known constant $A$. The first condition is $|Y_{t}^k| \leq A$ for any $k,t$ and the second condition is simply $\|\bM^k\|_\infty \leq A$. This is a mild assumption in  applications such as recommendation systems, where ratings are  bounded.

\section{{\malocate} algorithm}

\label{s:alg}We now describe our active strategy \malocate for the active multiple matrix completion given as Algorithm~\ref{alg:malocate}. The input for \malocate is the maximum budget input~$n$ and the loss parameter $p$. This parameter defines which loss $\cL^p_n$ the strategy should optimize for. We shall see that $p$ governs the exploration. During the initialization, while $B_k(t) = \infty$, the strategy requests for each~$\bM^k$ a dataset $\cD^k_t$ of size $\cO(d_k \log d_k)$. 
\malocate uses the requested samples for two
goals: computing the estimators \emph{and}
adaptively estimating their error. 
In particular, the first half of the requested sample is used to compute an estimator $\widehat{\bM}^k_t$ of $\bM^k$ using the square-root lasso estimator. The second half of the sample is used by the \esterr$(\widehat{\bM}^k_t, \cD^k_t)$ sub-procedure to construct an estimator of the error $\widehat{R}_{N_k^t}$ and an \textit{upper-bound} on this error $B_k(t)$, using the \emph{double-sampled} entries. After the initialization, at round $t$, the strategy allocates the next samples to the matrix 
\[m \triangleq \arg\max_k d_k^2 B_k(t) T_k(t)^{-1/p},\]
where $T_k(t)$ is the number of samples allocated to matrix $k$ up to round $t$.  
  The previous estimator $\widehat{\bM}^m$ for matrix $m$ is then replaced by $\widehat{\bM}^m_t$ \emph{only if}  the upper bound on the error has decreased. 
The strategy operates on a \emph{doubling schedule}: Each round an index $m$ is chosen, a new dataset $\cD^m_t$ of size $T_m(t)$ (and thus, a total budget of $2T_m(t)$ is spent on $m$) is used to construct a new estimator $\widehat{\bM}^m_t$, and estimate its error.

In this case, $B_m(t)$ is also updated to the new (smaller) upper bound on the error. This ensures that the estimation error is \textit{non-increasing} with $t$ for every matrix. This is a crucial ingredient for the proof of Theorem~\ref{thm:algoub}, which characterizes the performance of \malocate. 
The loop is repeated until the budget has been used, at which point the algorithm stops and outputs estimator~$\widehat{\bM}^k$ for each matrix $k$.

\begin{algorithm}[th]
\caption{\malocate algorithm}
   \label{alg:malocate}
\begin{algorithmic}
   \STATE {\bfseries Input:} $n$, $\{d_k\}_{k \in [K]}$, $p$ {\footnotesize \COMMENT{\emph{loss parameter}}} 
   \STATE $\cD^k_t \leftarrow \emptyset \quad \forall k \in [K]$
   \STATE \textbf{Initialization:}
   \bindent
   \STATE  $t \gets 0$
   \FOR{$k \in [K]$}
   \STATE $T_k(t) \gets 0$
   \STATE $B_k(t) \gets \infty$
   \ENDFOR
   \eindent
   \WHILE{$t \leq n$}
   \STATE $m \gets \arg\max_{k \in [K]} d_k^2 B_k(t)T_k(t)^{-1/p}$
   \STATE $T_m \gets  \max\pa{T_m(t),4\ceil{(d_m\log(d_m)+1)/2}}$
   \STATE $t  \gets t + T_m$ 
   \STATE $T_m(t) \gets T_m(t) + T_m$
   \STATE $\cD_t^m \leftarrow \newdata\big(m, T_m\big)$
   \STATE $\widehat{\bM}^m_t \gets \getest\big(m,\cD^m_t\big)$
   \STATE $N_t^m, \widehat{R}_{N_t^m} \leftarrow \esterr\big(\widehat{\bM}^m_t,\cD_t^m\big)$
   \STATE $B_k(t) \gets B_k(t-T_m) \quad \forall k \in [K]$
   \IF{$\widehat{R}_{N_m^t} + 8A^2\sqrt{\log(d_m)/N_t^m} \leq B_m(t)$}
   \STATE $\widehat{\bM}^m \leftarrow \widehat{\bM}^m_t$
   \STATE $B_m(t) \gets \widehat{R}_{N_t^m} + 8A^2\sqrt{\log(d_m)/N_t^m}$
   \ENDIF
   \STATE $T_k(t) \gets T_k(t-T_m) \quad \forall k \neq m$
   \ENDWHILE
   \STATE \textbf{Output:} $\{\widehat{\bM}^k\}_{k \in [K]}$
\end{algorithmic}
\end{algorithm}

\begin{algorithm}
\caption{\newdata$(k, T)$}
   \label{alg:samples}
\begin{algorithmic}
   \STATE {\bfseries Input:} $k, T$
   \bindent
   \STATE Sample uniformly at random \\
   \quad   \quad $T$ new observations $\{(X_i, Y_i)\}_{i \leq T}$ from $\bM^k$
   \eindent
   \STATE \textbf{Output:} New dataset $\{(X_i, Y_i)\}_{i \leq T}$ 
\end{algorithmic}
\end{algorithm}
\paragraph{Computing the estimator} As explained previously, we use the square-root lasso estimator. Notice that we perform a splitting of the sample $\cD^k_t$, where the first half is used to compute the estimator, and the second half is used to estimate its error. In practice, we propose instead to split the sample between entries that have been sampled only once to compute the estimator, and the other entries to estimate the error. While this introduces a small dependence (as we may only estimate the error for entries on which the estimator was not trained) which is difficult to analyze, in practice, this greatly improves the power of the estimator.

\paragraph{Estimating the error} The sub-procedure $\esterr$ uses the second half of a dataset $\cD^k_t$ to build an estimator of the error for some estimator~$\widehat{\bM}^k$ of the matrix $\bM^k$. It proceeds as the estimator of~\cite{carpentier2017adaptive} by finding entries $(X_i, Y_i)$ and $(X_j, Y_j)$ such that $X_i = X_j$ to form the triplet $(X_i, Y_i, Y_j)$, and the dataset $\cD'$ of double-sampled entries with $N^k_t \triangleq |\cD'|$. $\cD'$ is then used to compute the unbiased estimator of the error, \[\widehat{R_N} \triangleq  \frac{1}{N}\sum_{i = 1}^{N} \pa{Y_i - \inner{X_i}{\widehat{\bM}}}\pa{Y'_i - \inner{X_i}{\widehat{\bM}}},\] which \textit{does not require the variance} of the noise as an input to the estimation procedure. 
We can then deduce an upper bound on $\widehat{R_N}$ that holds with high probability $B_k(t) \triangleq \widehat{R_{N^k_t}} + 8A^2\sqrt{\log(d_k)/N^k_t}$. Importantly, this upper bound on the error is \emph{honest} and \emph{adaptive} to the unknown rank $r_k$ as proved by~\cite{carpentier2017adaptive} and is upper bounded as $\cO\pa{r_k d_k^3 \log(d_k)/T_k(t)}$, as $\widehat{R_{N_t^k}}$ dominates the stochastic error term.

\begin{algorithm}[th]
\caption{\getest$(k, \cD)$}
   \label{alg:est}
\begin{algorithmic}
   \STATE {\bfseries Input:} $k, \cD$
    \bindent
   \STATE $T \gets|\cD|/2, \lambda \gets C\sqrt{\log(d_k)/d_kT}$
   \STATE $\widehat{\bM} \gets \mathop{\arg\min}\limits_{\norm{\bM}_\infty \leq A} \sqrt{\frac{1}{T}\sum_{i = 1}^{T} (Y_i - \inner{X_i}{\bM})^2} +\lambda \norm{\bM}_\star$
      \eindent
   \STATE \textbf{Output:} Estimator $\widehat{\bM}$
\end{algorithmic}
\end{algorithm}

\paragraph{The sampling criterion} The exploration crucially depends on the interplay between  the loss parameter $p$, $T_k(t)$, and the upper bound on the error $B_k(t)$ rescaled by $d_k^2$. For $p=1$ (sum loss), the chosen index is \[\arg\max_k d_k^2 B_k(t) T_k(t)^{-1},\] and can be interpreted as the index that maximizes the error per sample, which is a rough approximation of $\partial B_k(t)/\partial T_k(t)$. The idea behind this heuristic is that since we expect the sum loss to decrease the most for this matrix, the next sample is allocated to this index. On the other hand, for $p = \infty$, the index chosen is simply the one that currently suffers the largest upper bound on the rescaled error.

\begin{algorithm}[th]
\caption{\esterr$(\widehat{\bM}, \cD)$}
   \label{alg:error}
\begin{algorithmic}
   \STATE {\bfseries Input:} $\widehat{\bM}, \cD$
       \bindent
   \STATE $T = |\cD|/2$
   \STATE Find double-sampled entries \\
   \quad   \quad  $\cD' \gets \{(X_i, Y_i, Y_i')\}_{i = 1, ..., N}$ in $\cD_{T+1, ..., 2T}$ 
\STATE $\widehat{R_N} \gets \frac{1}{N}\sum_{i = 1}^{N} \pa{Y_i - \inner{X_i}{\widehat{\bM}}}\pa{Y'_i - \inner{X_i}{\widehat{\bM}}}$
      \eindent
   \STATE \textbf{Output:} Number of double-sampled entries $N$ and\\ \hspace{1.45cm} error estimate $\widehat{R_N}$
\end{algorithmic}
\end{algorithm}
More generally, by plugging the upper bound given by Proposition~\ref{thm:ub} into the loss $\cL^p_n$, we see that a good allocation is one that minimizes \[\sum_k \pa{\frac{r_k d_k^3\log d_k }{T_k(n)}}^{\!p}\] under the constraint $\sum_k T_k(n)=n$. 
By solving the corresponding optimization problem, we see that this good allocation should be such that \[T_k(n)^{1+1/p} = (r_k d_k^3\log d_k )C(n),\] where $C(n)$ is constant for all $k$. Note however, that this good allocation is de facto out of reach for the learner, which does not have access to the underlying ranks $\{r_k\}_{k\in[K]}$ of the matrices. Now, as $d_k^2 B_k(t)$ can be upper bounded as $\cO\pa{r_k d_k^3 \log d_k /T_k(t)}$, it is clear that our strategy, which picks the index that maximizes $d_k^2B_k(t)T_k(t)^{-1/p}$ mimics the good allocation that keeps the quantity \[r_k d_k^3 \log(d_k) T_k(n)^{-(1+1/p)}\] constant across the arms.

\begin{remark}
An important algorithmic particularity of our strategy is that it operates on a doubling schedule. Namely, when index $k$ is picked, the number of observations for $\bM^k$ is doubled from $T_k(t)$ to $2T_k(t)$, as a new dataset of size $T_k(t)$ is generated. This allows us to analyze {\malocate} without considering correlations between the different estimators, as each estimator is trained on a fresh sample $\cD^k_t$. This also has the benefit of greatly reducing the computational complexity, as we only need to train a logarithmic number of estimators, while recomputing estimators at each round $t$ would be too costly. However, if there 
is an empirical need to recalculate 
the estimator every round we received a new observation, the proofs for the guarantee
that we provide in the next section 
can be modified to reflect it.
\end{remark}

\section{Analysis}
\label{s:theory}
In this section, we give guarantees on the performance of \malocate for general $p$, and prove a lower bound in the case $p = \infty$, showing that our strategy is optimal for the max loss, up to logarithmic factors.
\subsection{Upper bound on the loss of \malocate}
We start with upper bounding the loss of \malocate that holds with high probability.
\begin{restatable}{thm}{restaalgoub}\label{thm:algoub}
After $n$ sample requests, {\malocate} started with loss parameter $p$ outputs $K$ estimators, such that with probability at least $1- \sum_k 16\log(d_k)/d_k,$
\begin{align*}
\cL_n^p &\triangleq \left(\sum_{k\in [K]} \norm{\hat{\bM}^k_n - \bM^k}_F^{2p}\right)^{\!\!\!1/p} \\&\leq \cO \left( \frac{\left(\sum_{k=1}^K (r_k d_k^3 \log d_k )^{\frac{p}{p+1}}\right)^{\!\!\!\frac{p+1}{p}}}{n} \right)\cdot 
\end{align*}
\end{restatable}
We prove this result in Appendix~\ref{app:proof}. It relies on a careful bounding of the estimation error of $\widehat{\bM}_n$ directly, as it is \emph{not possible}\footnote{For example, if one of the estimators of $\bM^k$ is \textit{by chance} very good  despite having been given few samples, then it is possible that it will not be given more samples.} to prove bounds on $T_k(n)$, the number of times that each arm has been sampled at the end of the horizon, as opposed to many regret analyses used for bandit settings. In particular, the proof proceeds by showing that the following bounds on the error hold with high probability. First, using the sampling criterion we prove that for all $k$ a bound  of the form 
\begin{align*}&\norm{\widehat{\bM}^k-\bM^k}_F^{2} \\ &\quad\ \leq \cO \pa{T_k(n)^\frac{1}{p} \pa{\sum_k{ (r_k d_k^3\log d_k)^\frac{p}{p+1}}}^{\!\!\!\frac{p+1}{p}} \!\!\! n^{-\frac{p+1}{p}}}.\end{align*}
Importantly, this \textit{grows} with $T_k(n).$ On the other hand, Proposition~\ref{thm:ub} yields that \[{\norm{\widehat{\bM}^k-\bM^k}_F^2 \leq \cO \pa{\frac{r_kd_k^3\log d_k}{T_k(n)}}} \CommaBin\] which \textit{decreases} with $T_k(n)$. By balancing both bounds with respect to $T_k(n)$, we get an upper bound on the estimation error that does not depend on $T_k(n)$.

This result shows that the complexity of the problem crucially depends on the interaction between both the intrinsic difficulty of each sub-problem associated with~$\bM_k$, characterized by $r_k$ and $d_k$, and the loss parameter $p$. Namely, if we set $$c_k \triangleq r_k d_k^3 \log(d_k)$$ for the \emph{complexity} of problem~$k$, and $\bc = (c_1, \dots, c_K)$, then the complexity of the active problem is $\norm{\bc}_{\frac{p}{p+1}}$ i.e., the loss is upper bounded as 
\[\cO\pa{\norm{\bc}_{\frac{p}{p+1}}n^{-1}}.\]
On the other hand, it is easy to see that the uniform strategy suffers a loss of order $\frac{K}{n}\norm{\bc}_p$, which is always larger\footnote{as we have $\norm{\bx}_{q_1} \leq K^{1/q_1 - 1/q_2} \norm{\bx}_{q_2}$ for $0 < q_1 < q_2$} than $\frac{1}{n}\norm{\bc}_{\frac{p}{p+1}}$. This shows that our active strategy, \malocate, adapts on-the-fly to the difficulty of the problem at hand, without requiring any input parameter that depends on this complexity.

We now rewrite the previous theorem for the important case $p = \infty$.
\begin{corollary}{(upper bound for max loss)}
After $n$ sample requests, {\malocate} started with loss parameter $p=\infty$ outputs $K$ estimators, such that with probability at least $1- \sum_k 16\log(d_k)/d_k,$
\[\max_{k\in [K]} \norm{\hat{\bM}^k_n - \bM^k}_F^2 \leq \cO \pa{\frac{\sum_{k=1}^K r_k d_k^3 \log d_k}{n}}\cdot \]
\end{corollary}
This result is a direct corollary of our main upper bound. It shows that interestingly, even in the case $p = \infty$, the complexity of each \textit{individual} problem comes into play. Namely, in this setting, the total complexity is simply the sum of the complexities for each sub-problem.
\begin{remark}
While our results are stated in the fixed-budget setting, our strategy can easily be adapted to the $(\epsilon, \delta)$-correct setting, by slightly modifying the estimators, in particular by replacing $\log d_k$ terms by $\log(1/\delta)$ and re-deriving the bounds on their performance. The sample complexity would be of order $\tilde{\cO}(\norm{c}_{\frac{p}{p+1}}\epsilon^{-1})$. Interestingly, in this setting, it is also possible  to  design a stopping rule, as we have adaptive confidence bands on the estimates of $\epsilon_t$, the error at round $t$.
\end{remark}
\subsection{Lower bound}
We now show a lower bound for the active multiple matrix completion problem in the case $p = \infty$. The offline 
part of our lower bound proof is inspired by~\cite{koltchinskii2011nuclear}. The challenge of our proof is the  active setting as we have to consider strategies that may actively spread their observations over the different matrices.
\begin{restatable}{thm}{restalb}\label{thm:lb_max}
For any active strategy $\cS$, there exists a problem $P = (\bM^1, \dots, \bM^K)$, where $\bM^k$ is of rank at most $r_k$ and dimension $(d_k \times d_k)$, such that after $\cS$ (actively) collects at most $n$ observations corrupted with $\cN(0, A^2)$ noise and outputs $K$ estimators $(\widehat{\bM}^1, \dots, \widehat{\bM}^K)$, we have
\[
\mathbb E_{P,\cS}\left[\max_{k \in [K]}\pa{\norm{\widehat{\bM}^k - \bM^k}_F^2}\right] \geq \frac{A^2}{2048}\frac{\sum_{k=1}^K r_k d_k^3}{n}\cdot
\]
\end{restatable}
We prove this theorem in Appendix~\ref{s:proof_lb}. The main argument is that for any active strategy $\cS$, for any fixed problem $P$, there exists one index $m \in [K]$ such that \[\mathbb E_{P,\cS}\left[T_k(n)\right] \leq \frac{r_md_m^3}{\sum_k r_k d_k^3}n.\]
Then, we carefully adapt  the arguments of the lower bound for $K=1$ to our active setting.

This shows that our active strategy is minimax-optimal (up to logarithmic factors) over the class of problems with dimension $\{d_k\}_{k \in [K]}$ and ranks at most $\{r_k \}_{k \in [K]}$, fully adaptive to the unknown ranks of the sub-problems. Importantly, the lower bound also holds for strategies that have a priori knowledge of $\{r_k \}_{k \in [K]}$.

\begin{remark}\label{rem:anyestimator}
Notice, that while Algorithm~\ref{alg:est} uses a particular square-root lasso estimator with associated guarantees, 
our approach straightforwardly extends to other estimators. 
For example, \citet{klopp2015matrix}
provides sharp bounds in the Bernoulli model, i.e., without the extra $\log d_k$ factor. Therefore, 
this or any other result, that provides a sharper estimator
could be used instead in Algorithm~\ref{alg:est}. This would improve the overall complexity of our active strategy by removing the extraneous $\log d_k$ factors in the complexity, matching exactly the lower bound for $p = \infty$. 
\end{remark}

\section{Synthetic experiments}
\label{s:exp}
We now support our analysis \malocate with synthetic experiments. To create a square matrix of rank $r$ and dimension $d$, we generate two matrices $\bU \in \mathbb R^{d \times r}$ and $\bV \in \mathbb R^{r \times d}$ with entries distributed as $\cN(0, \sigma_r^2 \triangleq r^{-1/2})$. The standard deviation $\sigma_r$ is chosen such that the entries of $\bM = \bU \bV$ have the same scaling, regardless of the rank of the matrix. Observations are corrupted with Gaussian white noise $\cN(0, \sigma \triangleq 0.1)$. We consider both objectives $\cL_p$ for $p=1$ and $p=\infty$, on which we run \malocate also with both parameters $p=1$ and $p=\infty$. We also compare \malocate to the na\"ive uniform strategy, and for the max loss also with the oracle strategy that has access to the true Frobenius error of the estimators and allocates the next samples to the index $\arg\max_k \|\widehat{\bM}^k_t - \bM\|_F^2$. Note that this strategy (for a fixed estimation procedure) is optimal for $p = \infty$, as the max loss may only decrease if the worst estimator is improved. 

As our goal is to study the active advantage of \malocate, all the strategies have access to the same estimator $\texttt{SoftImpute}$, tuned with the same parameters. Moreover, we  discretize time in a similar fashion for all the strategies: The initialization phase of each estimator is done with $8d_k$ samples and after that, the budget is divided evenly in approximately $100$ sub-samples. This allows to bypass the negative effects associated with a doubling schedule. As our strategy is naturally anytime, we plot the results as the time horizon grows from the initialization up to $n = Kd^2/2$. At each round $t$ where a new estimator has been trained, we use the knowledge of $\bM^k$ to compute $\cL^p_t$ for $p \in \{1, \infty\}$. For both experiments, we draw and fix the problem, and average the results over $15$ runs.

\paragraph{First experiment}  We fix $d_k \triangleq d \triangleq 200$, $K \triangleq 10$, and the ranks are such that $r_k \triangleq 10$ for all $k$ besides $r_1 = 40$. We choose this instance as it forces the strategy into a tradeoff with respect to the loss parameter~$p$. Heuristically, to optimize the sum loss ($p=1$), reaching a good error on each of the easy problems is very important. On the other hand, to optimize the max loss, it is necessary to spend a large portion of the budget on the hardest instance. 
In Figure~\ref{f:f}, we see  that our strategies perform favorably in the setting they are  designed for. We also see that the uniform strategy only catches up when the number of samples is high enough such that the careful sample allocation has little effect on the performance.
\begin{figure}[t]
\centering
\subfigure{
\label{fig:first}
\includegraphics[height=2.4in]{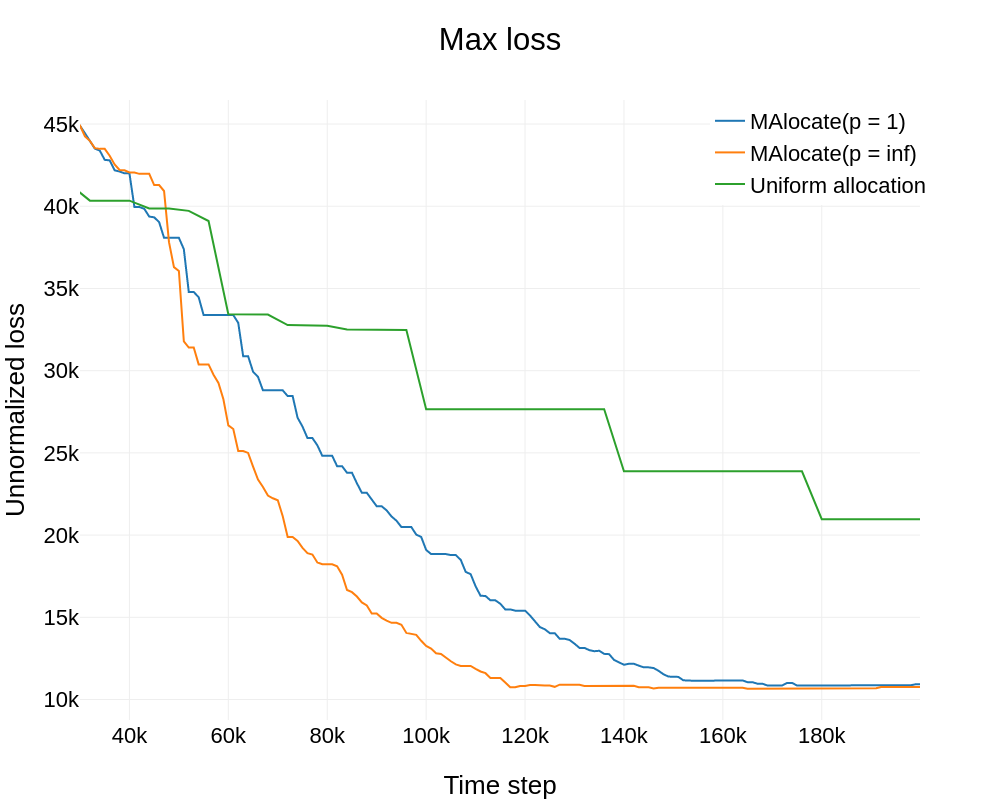}}
\qquad
\subfigure{
\label{fig:second}
\includegraphics[height=2.4in]{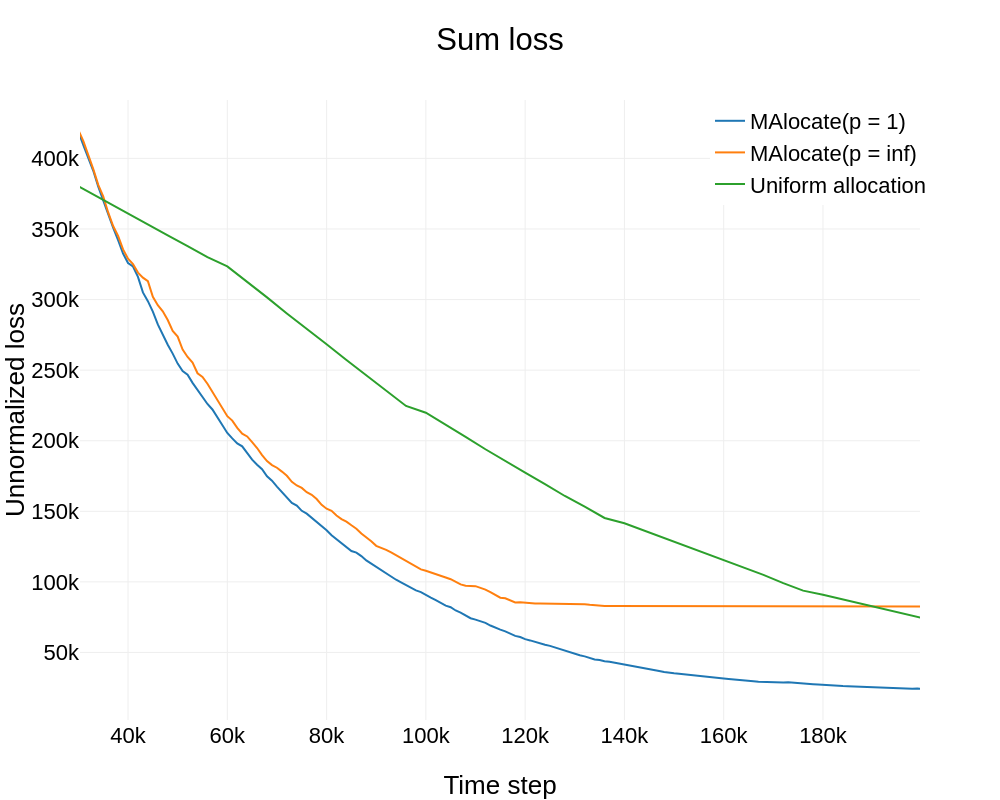}}
\caption{Results for the first experiment}
\label{f:f}
\end{figure}
\paragraph{Second experiment} We fix $d_k \triangleq d \triangleq 200$ and $K \triangleq 15$. The ranks $r_k$ are given by $r_k \triangleq 18 + 0.0015 k^4$. Note that the hardest instance is such that $r_{15} = 76$ and half of the sub-problems have rank at most $22$. This set of problems is more varied than the previous one and shows the adaptivity of our strategy (Figure~\ref{f:s}).

\begin{figure}[th]
\centering
\subfigure{
\label{fig:first_exp2}
\includegraphics[height=2.4in]{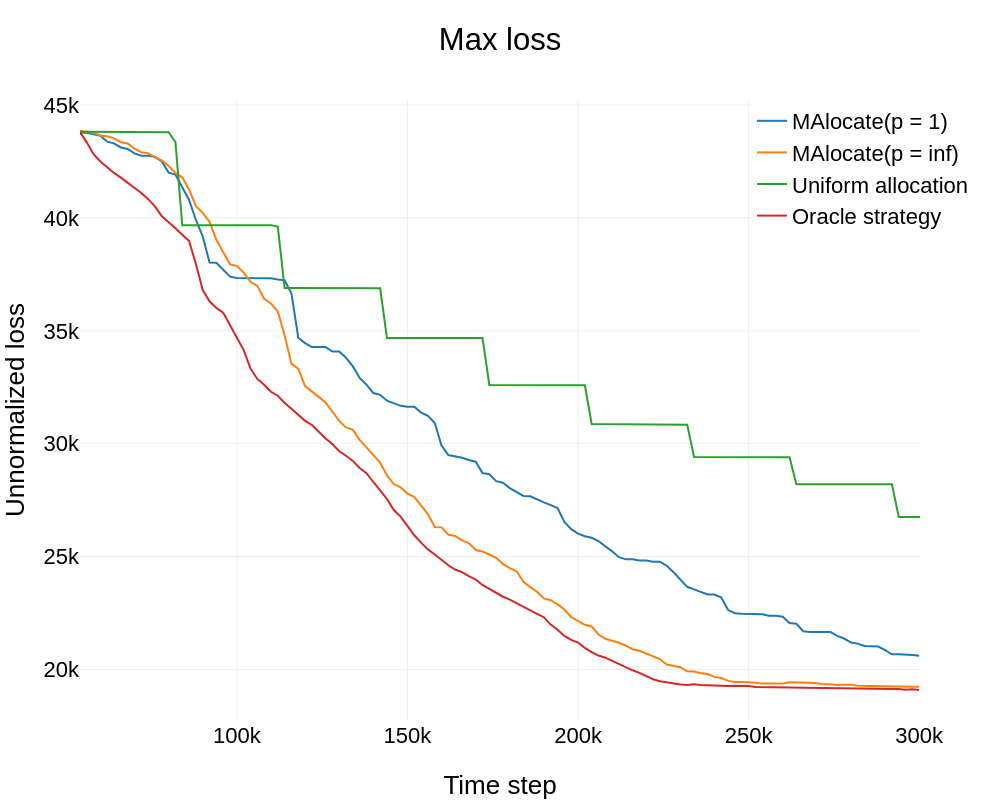}}
\qquad
\subfigure{
\label{fig:second_exp2}
\includegraphics[height=2.4in]{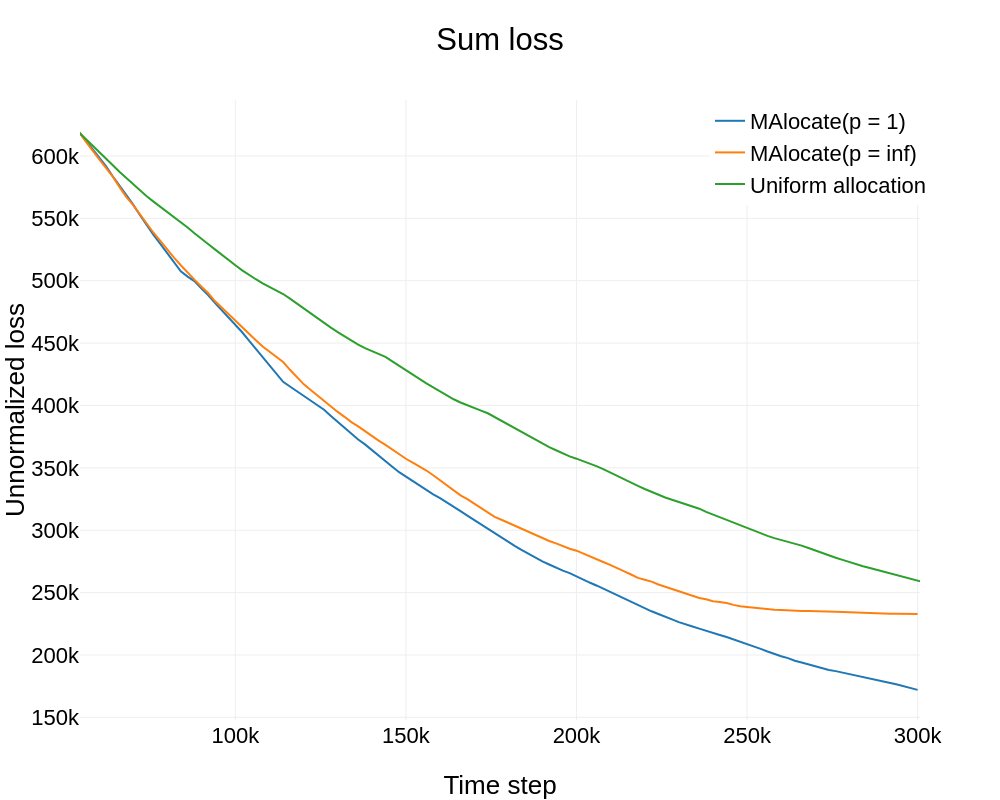}}
\caption{Results for the second experiment}
\label{f:s}
\end{figure}

\paragraph{Implementation of \malocate} 
As we discuss in {Remark~\ref{rem:anyestimator}}, our generic strategy can be used for \textit{any estimator}, which may be chosen appropriately with respect to the exact noise setting.
For performance reasons, we used the \texttt{SoftImpute}  estimator \citep{mazumder2010spectral} from the python package \texttt{fancyimpute}, which we tweak to have a warm-start heuristic that fills missing entries with the previous estimator $\widehat{\bM}^k$. This allows us to speed-up the running time. More generally, online matrix completion results such as the ones by~\cite{dhanjal2014online,lois2015online,jin2016provable} fit in our active and sequential framework. We tune the confidence intervals on the error in a conservative way. As we use a time discretization instead of a geometric grid, we also re-use samples throughout the run. Finally, as explained in Section~\ref{s:alg}, instead of splitting the entire sample, we use entries that have been observed once to train the estimator, and the other entries (sampled at least twice) to estimate the error.

Across the experiments, we see that \malocate run with the proper loss parameter $p$ indeed performs better on the associated loss $\cL^p$. For the max loss, we also see that \malocate with $p=\infty$ performs only slightly worse than the optimal oracle strategy in this setting. On the other hand, the uniform strategy performs poorly across the problems. We see that for the max loss, the loss peters out when the hardest matrix to estimate has been sampled $d_k^2$ times, as we cap the number of observations for each matrix to $d_k^2$. We remark however that we are interested in settings with smaller~$n \ll Kd_k^2$, where we see that \malocate with $p=\infty$ performs very favorably.

\section{Conclusion and discussion}
 We presented a new active matrix completion setting and provided \malocate, an active strategy that is able to adapt to the different 
complexities of the problems and proved that up to 
log factors, it achieves minimax-optimal guarantees. We also showed that empirically, it performs in accordance with its theoretical guarantees for two loss settings.
We see our work as the first step towards a more systematic understanding of the links between adaptive confidence sets (in any statistical setup) and the corresponding active learning setting.

We considered the \textit{high-dimensional} regime where the number of samples $n$ satisfies $d \leq n \ll d^2$. The number of doubly-sampled entries scales (w.h.p.) by Proposition~\ref{prop:alexlbtwice} as $n^2/d^2$ for any $n$ in this interval. This remains true for $n \gg d^2$ and generally our results would also hold in this regime. However, we
do not address this case here at all, as from an algorithmic point of view,  much simpler estimation strategies solve this problem, for example, least squares with a projection on the set of rank $r$ matrices coupled with Lepski's method to adapt to the rank.

Finally it is, unfortunately, not possible to extend our approach to datasets where entries are not observed twice, because it is \textit{provably impossible} to obtain a good estimator of the error.
\paragraph{Acknowledgements} 
The research presented was supported by European CHIST-ERA project DELTA, French Ministry of Higher Education and Research, Nord-Pas-de-Calais Regional Council, Inria and Otto-von-Guericke-Universit\"at Magdeburg associated-team north-European project Allocate, and French National Research Agency projects ExTra-Learn (n.ANR-14-CE24-0010-01) and BoB (n.ANR-16-CE23-0003). The work of A.\,Carpentier is also partially supported by the Deutsche Forschungsgemeinschaft (DFG) Emmy Noether grant MuSyAD (CA 1488/1-1), by the DFG - 314838170, GRK 2297 MathCoRe, by the DFG GRK 2433 DAEDALUS, by the DFG CRC 1294 Data Assimilation, Project A03, and by the UFA-DFH through the French-German Doktorandenkolleg CDFA 01-18.
This research has also benefited from the support of the FMJH Program PGMO and from the support to this program from Criteo.
\vfill

\bibliography{library}
\onecolumn
\appendix

\section{Upper bound for \malocate}\label{app:proof}
As explained in Section~\ref{s:setting}, in order to simplify the analysis, we only consider square matrices of dimension $d_k$ or $d$ below when we restate results for $K = 1$.

\begin{prop}[bound on estimation error, \citealp{klopp2012noisy}]
 \label{prop:olgalasso}
Consider the estimation problem in Frobenius norm for a matrix $\bM$ of rank $r$ with $n$ observations in the trace-regression model. $\bM$ is such that its entries, as well as the noisy observations of its entries are bounded by some (known) constant~$A$. Then, there exist numerical constants $c$ and $C$ such that the square root matrix lasso estimator $\widehat{\bM}_n$ satisfies with probability at least $1-3/d-2\exp(-cn)$
\[
\frac{\|\widehat{\bM}_n - \bM\|_F^2}{d^2} \leq CA^2 \cdot \frac{rd\log d}{n}\CommaBin
\]
where $\widehat{\bM}_n$ is defined as the solution to the following minimization problem,
\[
\widehat{\bM}_n \eqdef \argmin_{\norm{\bM}_\infty \leq A}\left\{\sqrt{\frac{1}{n}\sum_{i=1}^n \left(Y_i - \inner{\bX_i}{\bM}\right)^2} + \lambda \norm{\bM}_* \right\},
\]
with $\lambda \eqdef C'\sqrt{\log(d)/(dn)}$ and $C'$ is a numerical constant.
\end{prop}

\begin{prop}[concentration bound for $\widehat{R_N}$ estimator, \citealp{carpentier2017adaptive}] 
\label{prop:alexconcentration}
Let $\widehat{R_N}$ be the estimation procedure (sub-procedure {\esterr}) of $\|\widehat{\bM}-\bM\|_F^2$, for some $\widehat{\bM}$. Then, with probability at least $1-\frac{2}{d}\CommaBin$ we have
\[
\left|\widehat{R_N} - \frac{||\widehat{\bM}-\bM||_F^2}{d^2}\right| \leq 8A^2\sqrt{\frac{\log d }{N}}\cdot
\]
\end{prop}
\begin{prop}[Lower bound on the number of the entries sampled twice, \citealp{carpentier2017adaptive}]\label{prop:alexlbtwice}
For $n \leq d^2$, we have with probability at least ${1-\exp(-n^2/(372d^2))}$ that the number of entries sampled twice in a dataset of size $n/2$ is at least
\[N \geq \frac{n^2}{64d^2}\cdot\]
\end{prop}
We now define favorable events for which the estimators are within their confidence bounds for all datasets $\cD_t^k$, estimators $\widehat{\bM}_t^k$, and errors $\widehat{R_{N_t^k}}$ for well chosen rounds $t$, where $N_t^k$ is the number of entries sampled twice in the second half of the sample $\cD_t^k$. For $d_k \log d_k \leq t \leq d_k^2$, we write $\xi_1(t,k)$ for the event when these three bounds hold simultaneously,
\begin{align*}
 &(1)\quad \frac{\|\widehat{\bM}_t^k - \bM^k\|_F^2}{d_k^2} \leq CA^2 \cdot \frac{r_k d_k\log d_k}{t}\CommaBin\\
 &(2)\quad N_t^k \geq \frac{t^2}{64d_k^2} \CommaBin \\
 &(3)\quad\abs{\widehat{R_N} - \frac{\|\widehat{\bM}_t^k-\bM^k\|_F^2}{d_k^2}} \leq 8A^2\sqrt{\frac{\log d_k}{N_t^k}} \cdot
\end{align*}
Then we consider the following event $\xi_2(k)$,
\[
\xi_2(k) = \bigcap_{s \in [2\log_2(d_k)]} \xi_1(2^s T_k^{I}, k),
\quad \text{where} \quad
T_k^I \triangleq 2\ceil{\frac{d_k \log(d_k)+1}{2}}\cdot
\]
\begin{lem} For any $k \in [K]$, 
$\xi_2(k)$ does not hold with probability at most \[2\log_2(d_k)\left(\frac{5}{d_k}+2\exp(-c d_k\log(d_k))+\exp\pa{-\frac{\log^2 d_k}{372}}\right)\]

\end{lem}
\begin{proof}
The claim is consequence of a union bound using the
claims in Propositions~\ref{prop:olgalasso}, 
\ref{prop:alexconcentration}, \ref{prop:alexlbtwice},
together with $2d_k\log d_k \leq t \leq d_k^2.$
\end{proof}
\restaalgoub*
\begin{proof}
We consider $\xi_3 = \bigcap_{k \in [K]} \xi_2(k)$, which holds with probability at least \[1 - 2\sum_k \log_2(d_k) \pa{\frac{5}{d_k}+2\exp(-c d_k\log d_k)+\exp\pa{-\frac{\log^2 d_k}{372}}}.\]
The rest of the proof is conditioned on the fact that $\xi_3$ holds.
The initialization phase, when $B_k(t) = \infty$ and each matrix sampled for the first time by the algorithm, is such that $\bM^k$ is sampled $2T_k^I$ times, where $T_k^I$ is set such that it is the smallest even integer strictly greater than $d_k \log d_k$. By definition, we have $2d_k\log d_k \leq 2T_k^I \leq 4 d_k \log d_k.$ We remark here that $2T_k^I \geq 2d_k\log d_k$ ensured that on $\xi_3$, there is at least one double entry in the second half of the sample after the first time that matrix $k$ is sampled, since \[\frac{(2T_k^I)^2}{64d_k^2} \geq \frac{\log(d_k)^2}{16} \geq 1\]
for $d_k \geq 55$. This ensures that the $B$-values are finite as soon as the matrices have been sampled $2T_k^I$ times during the initialization.

For $n \geq 48\sum_{k \in [K]} d_k \log d_k = 12 \sum_{k} T_k^I$, there necessarily exists (by the pigeonhole principle) $m \in [K]$ such that $T_m(n)$ the total budget spent on matrix $m$ by the algorithm satisfies:
\begin{align*}
T_m(n) - 6 T_m^I  \geq  \frac{(r_m d_m^3 \log d_m)^{\frac{p}{p+1}}}{\sum_{k \in [K]} (r_k d_k^3 \log d_k)^{\frac{p}{p+1}}}\left(n - 6 \sum_{k \in [K]} T_k^I\right)
 \geq  \frac{(r_m d_m^3 \log d_m)^{\frac{p}{p+1}}}{\sum_{k \in [K]} (r_k d_k^3 \log d_k)^{\frac{p}{p+1}}} \pa{\frac{n}{2}}\cdot
\end{align*}
As the first two times that $k$ is chosen contribute $6 T_k^I \leq 12d_m \log d_m$ to $T_m(n)$, we know that $m$ is picked at least twice by the algorithm, and not just only during the initialization. For this $m$, we have ${T_m(n) \geq \frac{c_m}{\sum_{k} c_k} \left(\frac{n}{2}\right)}$, where we write for simplicity $c_k \eqdef (r_k d_k^3 \log d_k)^{\frac{p}{p+1}}$ with $r_k \eqdef \text{rank}(\bM^k)$.

We denote $t_1 < n$, the last round that the matrix $m$ was chosen by the algorithm. Since $t_1$ is the last round that matrix $m$ is chosen, and the algorithm operates on a doubling schedule, we have $T_m(t_1) = \frac{T_m(n)}{2} \geq \frac{c_m}{\sum_{k} c_k} \left(\frac{n}{4}\right)$. As we have established that matrix $m$ has been chosen at least twice by the algorithm, let us denote $t_2$ the penultimate round that matrix $m$ was chosen by the algorithm. By the same doubling reasoning, we have $T_m(t_2) \geq \frac{c_m}{\sum_{k} c_k} \left(\frac{n}{8}\right)$, and $\widehat{\bM}_{t_2}^m$ is such that the $B$-value for $m$ at round $t_1$ (which is non-increasing due the the definition of the algorithm) satisfies
\begin{align}\label{eq:bound_d2B_m}
d_m^2 B_m(t_1) = d_m^2 B_m(t_2+T_m(t_2)) & \leq d_m^2\left(\widehat{R}_{N_m^{t_2}} + 8A^2\sqrt{\frac{\log d_m}{N_m^{t_2}}}\right) \nonumber \\
& \leq  d_m^2\left(\norm{\widehat{\bM}_{t_2}^m - \bM^m}_F^2 + 16A^2\sqrt{\frac{\log d_m}{N_m^{t_2}}}\right) \nonumber \\
& \leq d_m^2\left(CA^2\cdot\frac{r_m d_m \log d_m}{T_m(t_2)} + 128A^2\frac{d_m\sqrt{\log d_m}}{T_m(t_2)}\right) \nonumber \\
& \leq  A^2\max(C, 128) \left(\frac{r_m d_m^3 \log d_m }{T_m(t_2)}\right),
\end{align}
where we use that on $\xi_3$, we have \[\hat{R_{N_m^{t_2}}} \leq \norm{\widehat{\bM}_{t_2}^m - \bM^m}_F^2 + 8A^2\sqrt{\frac{\log d_m}{N_m^{t_2}}}\] (in the second line) and $N_m^{t_2} \geq \frac{T_m(t_2)^2}{64d_m^{2}}$ (in the third line). Finally, we use $r_m \geq 1$ to get the ultimate line, as $r_m d_m \log d_m$ always dominates $d_m \sqrt{\log d_m}$. Now, plugging the lower bound on $T_m(t_2) \geq \frac{c_m}{\sum_{k} c_k} \left(\frac{n}{8}\right)$ brings
\begin{eqnarray}\label{eq:bound_B_m}
\frac{d_m^2 B_m(t_1)}{T_m(t_1)^{1/p}} & \leq & A^2\max(C, 128) \left(\frac{r_m d_m^3 \log d_m}{T_m(t_2)T_m(t_1)^{1/p}}\right)\nonumber \\
& = &  2^{1/p} A^2 \max(C, 128) \left(\frac{r_m d_m^3 \log d_m }{T_m(t_2)^{\frac{p+1}{p}}}\right)\nonumber \\
& \leq & 2^{1/p} 64A^2\max(C, 128) \left(\frac{\sum_k{c_k}}{n}\right)^{\frac{p+1}{p}}
\end{eqnarray}
At $t_1$, when matrix $m$ was chosen for the ultimate round, we had for all $i \neq m$,
\[
\frac{d_i^2B_i(t_1)}{T_i(t_1)^{\frac{1}{p}}} \leq \frac{d_m^2B_m(t_1)}{T_m(t_1)^{\frac{1}{p}}} < \infty,
\]
therefore all matrices $i$ had already been pulled at least once during the initialization. Combined with~\eqref{eq:bound_B_m}, this yields
\begin{equation}\label{eq:bound1_B_i}
d_i^2 B_i(t_1) \leq  2^{1/p} 64 A^2\max(C, 128) T_i(t_1)^{\frac{1}{p}} \left(\frac{\sum_k{c_k}}{n}\right)^{\frac{p+1}{p}}\!\!\!\cdot
\end{equation}
As $i$ has been sampled at least once, let us denote $t_i - \frac{T_i(t_1)}{2}$ the last round it was sampled before the round $t_1$. The following also holds, as the $B$-values are non-increasing with time (by design of the algorithm), and we have $T_i(t_1) = 2T_i(t_i)$,
\begin{align}\label{eq:bound2_B_i}
B_i(t_1) \leq B_i(t_i) & \leq  \widehat{R}_{N_i^{t_i}} + 8A^2\sqrt{\frac{\log d_i}{N_i^{t_i}}} \nonumber \\ 
& \leq  \norm{\widehat{\bM}^{t_i}_i - \bM^i}_F^2 + 16A^2\sqrt{\frac{\log d_i}{N_i^{t_i}}} \nonumber \\ 
& \leq  CA^2\left(\frac{r_i d_i \log(d_i)}{T_i(t_i)}\right) + 16A^2\sqrt{\frac{\log(d_i)}{N_i^{t_i}}} \nonumber \\
& \leq  CA^2\left(\frac{r_i d_i \log(d_i)}{T_i(t_i)}\right) + 128A^2\frac{d_i \sqrt{\log(d_i)}}{T_i(t_i)} \nonumber \\ 
& \leq  2A^2\max(C, 128)\left(\frac{r_i d_i \log(d_i)}{T_i(t_1)}\right)\cdot
\end{align} 
Finally, it is easy to see that as $B_i(t)$ cannot increase with $t$ and since the estimator $\widehat{\bM}^i$ is only updated if the error decreases, then for all $t$ we have ${\norm{\widehat{\bM}^i_n - \bM^i}_F^2 \leq d_i^2 B_i(t)}$ where we denote the final estimator output at round $n$ by the algorithm as $\widehat{\bM}^i_n$.
Combined with~\eqref{eq:bound2_B_i} this yields 
\[
\norm{\widehat{\bM}^i_n - \bM^i}_F^2 \leq 2A^2\max(C, 128)\left(\frac{r_i d_i^3 \log(d_i)}{T_i(t_1)}\right),
\]
which decreases with $T_i(t_1)$, and on the other hand, \eqref{eq:bound1_B_i} brings
\[
\norm{\widehat{\bM}^i_n - \bM^i}_F^2 \leq 2^{1/p}64 A^2\max(C, 128) T_i(t_1)^{\frac{1}{p}} \left(\frac{\sum_k{c_k}}{n}\right)^{\frac{p+1}{p}}\CommaBin
\]
which increases with $T_i(t_1)$. By combining both bounds, we get
\[
\norm{\widehat{\bM}^i_n - \bM^i}_F^2 \leq 2^{1/p} 64A^2\max(C, 128) \min\left(\frac{r_i d_i^3 \log(d_i)}{T_i(t_1)}, T_i(t_1)^{\frac{1}{p}} \left(\frac{\sum_k{c_k}}{n}\right)^{\frac{p+1}{p}}\right)\CommaBin
\]
and by maximizing this bound with respect to $T_i(t_1)$, we get
\[
\norm{\widehat{\bM}^i_n - \bM^i}_F^{2p} \leq 2 \pa{64A^2\max(C, 128A^2)}^p\frac{(r_i d_i^3 \log(d_i))^{\frac{p}{p+1}}(\sum_k c_k)^p}{n^p}\cdot
\]
By~\eqref{eq:bound_d2B_m} this bound also holds for $m$, and by summing the errors we get

\begin{eqnarray*}
\cL_n^p & =  & \left(\sum_{k \in [K]} \norm{\widehat{\bM}^k_n - \bM^k}_F^{2p}\right)^{1/p} \\
& \leq & \cO \left(\frac{(\sum_k c_k)}{n} \left(\sum_{k = 1}^{K} c_k \right)^{1/p} \right) \\
& \leq & \cO \left(\frac{(\sum_k c_k)^{\frac{p+1}{p}}}{n} \right) \\
& \leq & \cO \left(\frac{\left(\sum_k (r_k d_k^3 \log(d_k))^{\frac{p}{p+1}}\right)^{\frac{p+1}{p}}}{n} \right)
\end{eqnarray*}%
\vspace{-2em}
\end{proof}
\section{Lower bound for max loss ($p = \infty$)}\label{s:proof_lb}
\restalb*
\begin{proof}
The purpose of this lower bound is to show that for any active and possibly randomized strategy, there exists a problem on which it errs with constant probability, and that this error is of the same order as the upper bound we proved in Theorem~\ref{thm:algoub} for $p = \infty$. We begin by pointing out that although this lower bound holds for \emph{any} strategy, the construction hereunder depends on first fixing the strategy $\cS$. Our goal is to prove a lower bound over the class of problems denoted $\cP$ such that for any $P = (\bM^1, \dots, \bM^K) \in \cP$, $\bM^k$ is of dimension $(d_k \times d_k)$ and $\mathrm{rank}(\bM_k) \leq r_k$. At each round $t \leq n$, the strategy picks an index $k_t \in [K]$ and collects a noisy observation $Y_t = \inner{\bM^{k_t}}{X_t^{k_t}} + \epsilon_t$ where $\epsilon_t \sim \cN(0, A^2)$ and $X_t^{k_t}$ is taken uniformly at random. Although this is not exactly the noise model in which our upper-bound is stated, we use this for ease of notation, as all our results can be written instead with mean $1/2$ and $1/2 + \delta$. In particular, the centering in $0$ we use hereunder can be modified to $A/2$ to fit the bounded noise assumption by considering the distributions $0.5A\cB(1/2)$ and $0.5A\cB(1/2+\delta)$.

Let $\bM_k^0$ be the null matrix of size $(d_k \times d_k)$. We  refer to problem $0$ as the problem characterized by $(\bM^1_0, \dots, \bM^K_0)$. For the fixed strategy $\cS$, we define the quantity $\tau_k = \mathbb{E}_{0, \cS}[T_k(n)]$, where $T_k(n)$ is the number of observations from $\bM^k$ collected by strategy $\cS$ at the end of the active game. By definition of the fixed budget setting, we have $\sum_k \tau_k = n$.

We now define a set of problems for each matrix $\bM^k$. We write:
\[ 
\cR_k = \left\{ \tilde{\bM}^k = (m_{i,j}^k) \in \mathbb R^{d_k \times r_k}: m_{i,j}^k \in \left\{0, c A^2\sqrt{\frac{r_k d_k}{\tau_k}}\right\} \right\}\CommaBin
\]
where $c$ is a small numerical constant to be specified later. Importantly, any element of $\cR_k$ is of rank at most $r_k$. We now define
\[
\cM_k = \left\{\bM^k = \left( \tilde{\bM}_k \mid \dots \mid \tilde{\bM}_k \mid O \right) \in \mathbb{R}^{d_k \times d_k}, \tilde{\bM}_k \in \cR_k  \right\},
\]
where each matrix $\bM^k$ is just $\tilde{\bM}_k$ duplicated $\lfloor \frac{d_k}{r_k} \rfloor$ times, and the last few columns are completed by~$0$ entries to make the matrix square of dimension $d_k \times d_k$. By construction, this matrix has rank at most $r_k$, since the repeated pattern has rank at most $r_k$ itself.

By the Gilbert-Varshamov bound \citep{gilbert1952comparison,varshamov1957estimate}, we know that there exists a subset $\cB_k \subset \cM_k$, containing $\bM^k_0$, with cardinality at least $2^{r_kd_k/8}+1$ such that its elements are \emph{well separated}. Namely, for any two elements $\bM^k_i, \bM^k_j$  of $\cB_k$, we have
\[
\norm{\bM^k_i - \bM^k_j}_F^2 \geq \frac{c^2A^2}{16} \cdot \frac{r_k d_k^3}{\tau_k} \cdot
\]
We consider the set of problems $\cP_k = \left\{(\bM^1_0,\dots , \bM^k, \dots, \bM^K_0), \bM_k \in \cB_k \right\}$. We now define the distribution of the data (actively) collected under problem $i$ belonging to $\cP_k$ by strategy $\cS$ as $\mathbb P_{i,\cS}^n = \{(X_i^{k_i}, Y_i^{k_i})\}_{i \leq n}$ and write $\mathrm{KL}(\mathbb P_{j,\cS}^n, \mathbb P_{i,\cS}^n)$ for the Kullback-Leibler divergence between two such distributions. Using standard active learning arguments as used by~\citet[proof of Theorem 1]{castro2008minimax}, we have (using the sampling uniformly at random in the first line)
\begin{eqnarray*}
\mathrm{KL}\pa{\mathbb P_{0,\cS}^n, \mathbb P_{i,\cS}^n} & = \frac{1}{A^2} \sum_{k \in [K]} \norm{\bM^k_i - \bM^k_0}_F^2 \mathbb E_{0, \cS}(T_k(n))\\
& \leq \frac{c^2r_k d_k}{\tau_k} \mathbb{E}_{0,\cS}(T_k(n))\\
& \leq  c^2 r_k d_k \\
& \leq  \frac{c^2}{2}\log\pa{|\cP_k|},
\end{eqnarray*}

where we use in the second line that problems $i$ and $0$ in the class $\cP_k$ only differ on the $k$-th matrix. Taking $c = 1/2$, we have $\frac{1}{|\cP_k|}\sum_{i \leq |\cP_k|} \mathrm{KL}(\mathbb P_{0,\cS}^n, \mathbb P_{i,\cS}^n) \leq \alpha \log(|\cP_k|)$ for $\alpha = 1/8$. We can thus use Theorem~2.5 by~\cite{tsybakov2009introduction} on each set of problems $\cP_k$ with $s = \frac{A^2 r_k d_k^3}{128 \tau_k}\CommaBin$ where we write $\hat P = (\hat{\bM}^1, \dots, \hat{\bM}^K)$ for an estimator output by the active strategy $\cS$ on problem $P = (\bM^1, \dots, \bM^K)$:

\begin{eqnarray*}
\inf_{\hat P} \sup_{P \in \cP} \mathbb E_P\left(\max_k \pa{||\hat \bM^k - \bM^k||_F^2}\right) & \geq & \inf_{\hat P} \max_{k \in [K]} \sup_{P \in \cP_k} \mathbb E_P\left(\max_i (||\hat \bM^i - \bM^i||_F^2)\right)\\
& \geq & \inf_{\widehat{P}} \max_{k \in [K]} \sup_{P \in \cP_k} \mathbb E_P(||\hat \bM^k - \bM^k||_F^2)\\
& \geq & \max_{k \in [K]} \frac{A^2}{2048} \cdot \frac{r_k d_k^3}{\tau_k}\CommaBin
\end{eqnarray*}
where we lower bound $\frac{\sqrt{|\cP_k|}}{1+\sqrt{|\cP_k|}}\pa{1-2\alpha-\sqrt{\frac{2\alpha}{\log |\cP_k|}}}$ by $0.08$ for $|\cP_k| \geq 2$.

  Finally, by the pigeonhole principle, we know that for any (fixed) strategy $\cS$  there exists some index~$m$ such that $\mathbb E_{0, \cS}(T_{m}) = \tau_m \leq \frac{r_md_m^3n}{\sum_k r_k d_k^3}\CommaBin$ so we can lower bound: \[{\max_{k \in [K]} \frac{A^2}{2048} \cdot
\frac{r_kd_k^3}{\tau_k} \geq \frac{A^2}{2048} \cdot \frac{\sum_k r_k d_k^3}{n}}\cdot\]
\vskip -0.8cm
\end{proof}\end{document}